\def\year{2018}\relax
\documentclass[letterpaper]{article}
\usepackage{aaai18}
\usepackage{times}
\usepackage{helvet}
\usepackage{courier}
\usepackage{amssymb}
\usepackage{amsmath}
\usepackage{amsthm}
\usepackage{amsfonts}
\usepackage{cases}
\usepackage{color}
\usepackage{multirow}
\usepackage{tikz}
\usepackage{xifthen}
\usepackage{pifont}
\usepackage{ifthen}
\usepackage{graphicx}
\usepackage[linesnumbered,ruled,slide,boxed,vlined]{algorithm2e}

\newtheorem{definition}{Definition}[section]
\newtheorem{proposition}{Proposition}[section]

\newtheorem{theorem}{Theorem}[section]
\newtheorem{lemma}{Lemma}[section]
\newtheorem{example}{Example}

\newcommand{\cites}[1]{\citeauthor{#1} \shortcite{#1}}

\newcommand{\SV}{\mathsf{S5}}

\newcommand{\CNF}{\mathsf{CNF}}
\newcommand{\DNF}{\mathsf{DNF}}

\newcommand{\BDD}{\mathsf{BDD}}

\newcommand{\KDNF}{\mathsf{SDNF}}

\newcommand{\KFVDNF}{\mathsf{ASDNF}}

\newcommand{\KCNF}{\mathsf{SCNF}}

\newcommand{\KFVCNF}{\mathsf{ASCNF}}

\newcommand{\CDNF}{\mathsf{CDNF}}
\newcommand{\ACDNF}{\mathsf{ACDNF}}
\newcommand{\PINF}{\mathsf{PINF}}

\newcommand{\SAT}{\mathsf{SAT}}

\newcommand{\TE}{\mathsf{TE}}
\newcommand{\CL}{\mathsf{CL}}

\newcommand{\KTE}{\mathsf{STE}}

\newcommand{\KFVTE}{\mathsf{ASTE}}

\newcommand{\KCL}{\mathsf{SCL}}

\newcommand{\KFVCL}{\mathsf{ASCL}}

\newcommand{\propSub}{\mathsf{Prop}}

\newcommand{\Kn}{\mathsf{K_n}}

\newcommand{\KFVn}{\mathsf{K45_n}}
\newcommand{\KDFVn}{\mathsf{KD45_n}}
\newcommand{\SVn}{\mathsf{S5_n}}

\newcommand{\ALC}{\mathcal{ALC}}

\newcommand{\axiomF}{\mathbf{4}}
\newcommand{\axiomV}{\mathbf{5}}

\newcommand{\CO}{\mathbf{CO}}
\newcommand{\VA}{\mathbf{VA}}
\newcommand{\SE}{\mathbf{SE}}
\newcommand{\EQ}{\mathbf{EQ}}
\newcommand{\IM}{\mathbf{IM}}
\newcommand{\CE}{\mathbf{CE}}
\newcommand{\CT}{\mathbf{CT}}
\newcommand{\ME}{\mathbf{ME}}
\newcommand{\veeC}{\vee\mathbf{C}}
\newcommand{\veeBC}{\vee\mathbf{BC}}
\newcommand{\wedgeC}{\wedge\mathbf{C}}
\newcommand{\wedgeBC}{\wedge\mathbf{BC}}
\newcommand{\negC}{\neg\mathbf{C}}
\newcommand{\FO}{\mathbf{FO}}
\newcommand{\SFO}{\mathbf{SFO}}
\newcommand{\CD}{\mathbf{CD}}
\newcommand{\ACE}{\mathbf{ACE}}
\newcommand{\AIM}{\mathbf{AIM}}

\newcommand{\BoxSub}[1][]{B_{#1}}
\newcommand{\DiamSub}[1][]{D_{#1}}
\newcommand{\PropSub}{Prop}

\renewcommand{\P}{\mathbf{P}}
\newcommand{\NP}{\mathbf{NP}}
\newcommand{\coNP}{\mathbf{coNP}}

\newcommand{\PSPACE}{\mathbf{PSPACE}}

\newcommand{\Know}{\Box}
\newcommand{\CKnow}{\blacksquare}
\newcommand{\dualKnow}{\Diamond}
\newcommand{\cover}{\triangledown}

\newcommand{\lang}{\mathcal{L}}
\newcommand{\langkn}{\mathcal{L}_{\Know}}

\newcommand{\sublangprop}{\mathcal{L}_{0}}
\newcommand{\sublangpropP}{\mathcal{L}'_{0}}

\newcommand{\xmark}{\ding{53}}

\newcommand{\nmodels}{\not \models}
\newcommand{\nequiv}{\not \equiv}
\newcommand{\set}[1]{\{ {#1} \}}
\newcommand{\tuple}[1]{\langle {#1} \rangle}

\newcommand{\problem}{\mathcal{Q}}
\newcommand{\onticAct}{\mathcal{O}}

\newcommand{\epiAct}{\mathcal{E}}
\newcommand{\initialKB}{\mathcal{I}}
\newcommand{\goal}{\mathcal{G}}

\newcommand{\pre}{pre}
\newcommand{\eff}{e\!f\!f}

\newcommand{\pos}{pos}
\newcommand{\negation}{neg}

\newcommand{\dep}[1][]{\delta(#1)}
\newcommand{\len}[1][]{|#1|}

\newcommand{\prop}{P}
\newcommand{\subprop}{Q}
\newcommand{\agents}{\mathcal{A}}
\newcommand{\subagents}{\mathcal{B}}

\newcommand{\km}[1][]{\ifthenelse{\isempty{#1}} {\langle S,  R,  V \rangle} {\langle S_{#1}, R_{#1}, V_{#1} \rangle}}
\newcommand{\kmu}[1]{\ifthenelse{\equal{#1}{'}} {\langle S', R', V' \rangle} {\ifthenelse{\equal{#1}{''}} {\langle S'', R'', V'' \rangle} {\langle S^{#1}, R^{#1}, V^{#1} \rangle}}}

\newcommand{\kforget}[1][]{\mbox{kforget} \ifthenelse{\isempty{#1}} {} {(#1)}}

\newcommand{\bisimilar}[1][]{\underline{\leftrightarrow} \ifthenelse{\isempty{#1}} {} {_{#1}}}
\newcommand{\similar}[1][]{\sim \ifthenelse{\isempty{#1}} {} {_{#1}}}

\newcommand{\true}{\top}
\newcommand{\false}{\bot}

\newcommand{\ie}{{\it i.e.}}
\newcommand{\eg}{{\it e.g.}}

\renewcommand{\xmark}{\text{\ding{53}}}

\newcommand{\dense}{\addtolength{\itemsep}{-0.5mm}}

\title{Knowledge Compilation in Multi-Agent Epistemic Logics} 

\author{Liangda Fang$^{1}$ \hspace*{0.5cm}   Kewen Wang$^{2}$  \hspace*{0.5cm}  Zhe Wang$^{2}$  \hspace*{0.5cm} Ximing Wen$^{3}$\\
	$^{1}$Deptartment of Computer Science, Jinan University, China \\
	$^{2}$School of Information and Communication Technology, Griffith University, Australia \\
	$^{3}$Guangdong Institute of Public Administration, Guangzhou, China\\
	fangld@jnu.edu.cn, $\{$k.wang,zhe.wang$\}$@griffith.edu.au, wenxim@mail2.sysu.edu.cn\\}

\nocopyright
\begin{document}
	
	\maketitle
	
\begin{abstract}
	\looseness=-1
	Epistemic logics are a primary formalism for multi-agent systems but major reasoning tasks in such epistemic logics are intractable, which impedes applications of multi-agent epistemic logics in automatic planning.
	Knowledge compilation provides a promising way of resolving the intractability by identifying expressive fragments of epistemic logics that are tractable for important reasoning tasks such as satisfiability and forgetting.
	The property of logical separability allows to decompose a formula into some of its subformulas and thus modular algorithms for various reasoning tasks can be developed. In this paper, by employing logical separability, we propose an approach to knowledge compilation for the logic $\Kn$ by defining a normal form $\KDNF$.  
	Among several novel results, we show that every epistemic formula can be equivalently compiled into a formula in $\KDNF$, major reasoning tasks in $\KDNF$ are tractable, and formulas in $\KDNF$ enjoy the logical separability.
	Our results shed some lights on modular approaches to knowledge compilation.
	Furthermore, we apply our results in the multi-agent epistemic planning.
	Finally, we extend the above result to the logic $\KFVn$ that is $\Kn$ extended by introspection axioms $\axiomF$ and $\axiomV$.
\end{abstract}
	
	\section{Introduction}
	
	\looseness=-1
	It is crucial for an intelligent agent system to be capable of representing and reasoning about high-order knowledge in the multi-agent setting.
	A general representative framework for these scenarios is multi-agent epistemic logics. 
	However, many reasoning tasks in such logics are intractable, \eg, the entailment problems for $\Kn$ and $\KFVn$ are $\PSPACE$-complete \cite{HalM1992}.
	
	\looseness=-1
	These intractability results impede applications of multi-agent epistemic logics, \eg, multi-agent epistemic planning (MAEP) \cite{KomG2015,MuiBFMMPS2015}.
	An MAEP consists of a set of agents, the initial knowledge base (KB) and the goal formula that are expressed in multi-agent epistemic logics, ontic actions that change the world and epistemic actions that modify the mental attitude of agents.
	Two types of reasoning tasks, that are essential to solving MAEP, involve progression and entailment check.
	Progression updates KBs according to the effects of actions while entailment check is needed to decide if the current KB entails the goal formula and the preconditions of actions.
	As mentioned in \cite{BieFM2010}, based on a normal form with efficient progression and entailment procedures, the whole planning process should also be effective.
	
	\looseness=-1
	Knowledge compilation is an effective approach to address the intractability problem \cite{DarM2002}.
	A basic idea is to identify a normal form such that it is a fragment of the given language and each KB can be equivalently transformed into a KB in the normal form.
	\cites{BieFM2010} proposed a normal form, called $\SV$-$\DNF$, for the single-agent $\SV$ that supports polytime bounded conjunction and forgetting.
	It is also applied in making the progression of actions tractable.
	However, many reasoning tasks of multi-agent epistemic logics, including forgetting and entailment check, is intractable in this normal form.
	Hence, it cannot be applied to the multi-agent case.
	
	\looseness=-1
	Some normal forms have been proposed for multi-agent epistemic logics.
	By using cover operators instead of standard epistemic operators, \cites{CateCMV2006} defined \textit{cover disjunctive normal forms (CDNFs)} for $\mathcal{ALC}$ that is a syntactic variant of $\Kn$.
	\cites{Bie2008} introduced \textit{prime implicate normal forms (PINFs)} for $\ALC$.
	The target languages for these two compilations are tractable w.r.t. major reasoning tasks such as entailment check and forgetting.
	The former supports bounded conjunction while the latter does not.
	In the worst case, a compiled formula from CDNF has the single exponential size w.r.t. the original formula, but PINF can cause double exponential blowup. 
	In addition, a normal form, called \textit{alternating cover disjunctive normal form (ACDNF)}, is proposed for the logic $\KDFVn$ \cite{HalFD2012}.
	This form prohibits direct nestings of cover operators of an agent inside those of the same agent.
	Recently, \cites{HuangFWL2017} proved that polytime bounded conjunction and satisfiability check hold for ACDNFs.

	
	\looseness=-1
	To develop effective algorithms to MAEPs, we aim to develop a compilation approach for multi-agent logics such that (1) the compilation is relatively compact.
	That is, the compiled formula has at most single exponential size; (2) the target language is tractable for major reasoning tasks of MAEP: bounded conjunction, forgetting and entailment check; and (3) each formula can be equivalently transformed into a formulas in the normal form.
	
	\looseness=-1
	In this paper, we provide such a solution to knowledge compilation for the multi-agent epistemic logics $\Kn$ and its extension $\KFVn$, with the well-known introspection axioms $\axiomF$ and $\axiomV$, by employing the theory of \emph{logical separability} \cite{Lev1998}.
	Informally, we say a conjunction $\phi$ of formulas is logically separable if reasoning can be reduced to its conjuncts.
	For example, the formula $\phi = (p \rightarrow q) \land (q \rightarrow r)$ is not logically separable since it logically implies a conjunct $p \rightarrow r$ that is not derived by any single conjunct of $\phi$.
	By conjoining $\phi$ with the implicit conjunct, the new formula becomes logically separable.
	
	\looseness=-1
	The main contributions of this paper are summarized as follows:
	\begin{enumerate}
		\item \looseness=-1 
		We first formulate the concept of logical separability for epistemic terms and introduce some useful properties that are desired for them.		
		Thanks to the notion of logical separability, we are able to define two novel normal forms for $\Kn$, referred to as $\KDNF$ and $\KCNF$ (Section 3). 
		
		
		\item \looseness=-1 We provide an almost complete knowledge compilation map for multi-agent epistemic logics by comparing among the four normal forms: $\KDNF$, $\KCNF$, $\CDNF$ and $\PINF$ from the four aspects: expressiveness, succinctness, queries and transformations. 
		To the best of our knowledge, we are the first to construct this map for multi-agent epistemic logics (Sections 4 and 5).

		\item \looseness=-1
		We offer a tractable approach to progression and entailment checking, which are important ingredients of MAEP. 
		To achieve this, we obtain a normal form $\KDNF_{\sublangprop}$ by taking advantage of tractability of normal forms in propositional logic, \eg, $\DNF$ and $\BDD$ \cite{Bry1986} on bounded conjunction, forgetting and entailment check are tractable (Section 6). 
		
		
		\item \looseness=-1 We extend the results of knowledge compilation for $\Kn$ to $\KFVn$ by requiring that no consecutive epistemic operators of the same agent appears in formulas (Section 7). 
	\end{enumerate}
	
	
	\section{The multi-agent modal logic $\Kn$}
	
	\looseness=-1
	In this section, we first recall the syntax and semantics of the multi-agent epistemic logic $\Kn$, and then introduce two normal forms of $\Kn$, and major reasoning tasks in $\Kn$.
	
	\subsubsection{Syntax and semantics}	
	Throughout this paper, we fix a set $\agents$ of $n$ agents and a countable set $\prop$ of variables.	
	
	\begin{definition} \rm
		The language $\langkn$ is generated by the BNF:		
		\[\phi ::= \true \mid p \mid \neg \phi \mid \phi \land \phi \mid \Know_i \phi , \]		
		where $p\in \prop$ and $i \in \agents$.
	\end{definition}
	
	\looseness=-1
	The formula $\Know_i \phi$ means that agent $i$ knows $\phi$.
	The symbols $\false$, $\lor$, $\rightarrow$, $\leftrightarrow$, and $\dualKnow_i$ are defined as usual.
	We use $i$ and $j$ for agents, $\subagents$ for sets of agents, $p$ and $q$ for variables, $\subprop$ for finite sets of variables.
	For an $\langkn$-formula $\phi$, we use $\len[\phi]$ for the size of $\phi$ (\ie, the number of occurrences of variables, logical connectives, and modalities in $\phi$), $\dep[\phi]$ for the depth of $\phi$ (\ie, the maximal number of nested epistemic operators appearing in $\phi$), and $\prop(\phi)$ for the set of variables appearing in $\phi$.
	We say a formula $\phi$ is \textit{smaller} than $\psi$, if $\len[\phi] < \len[\psi]$.
	
	\looseness=-1
	The notions of propositional literals, terms ($\TE$), clauses ($\CL$), disjunctive and conjunctive normal forms ($\DNF$ and $\CNF$) are defined as usual.
	An $\langkn$-formula is in \textit{negation normal form (NNF)} if the scope of $\neg$ contains only variables.
	A \textit{positive} (resp. \textit{negative}) \textit{epistemic literal} is a formula of the form $\Know_i \phi$ (resp. $\dualKnow_i \phi$).
	A formula is \textit{basic}, if it is a propositional formula or epistemic literal.
	An \textit{epistemic term} (resp. \textit{clause}) is a conjunction (resp. disjunction) of basic formulas.
	Sometimes, we treat an epistemic term or clause as a set of formulas.
	For an epistemic term (resp. clause) $\phi$, we use $\propSub(\phi)$ for the set of the maximal propositional formulas that are conjuncts (resp. disjuncts) of $\phi$, $\BoxSub[i](\phi)$ for the set of formulas $\psi$ such that $\Know_i \psi$ is a conjunct (resp. disjunct) of $\phi$, and $\DiamSub[i](\phi)$ for the set of formulas $\psi$ such that $\dualKnow_i \psi$ is a conjunct (resp. disjunct) of $\phi$.

	\begin{definition} \label{def:model}\rm		
		A Kripke model $M$ is a tuple  $\km$ where
		\begin{itemize} 
			\item $S$ is a non-empty set of possible worlds;
			\item $R = \set{R_i \mid i \in \agents}$ where $R_i$ is a binary relation on $S$; 
			\item $V$ is a function assigning to each $s \in S$ in a subset of $\prop$.
		\end{itemize}
		
		A pointed Kripke model is a pair $(M, s)$, where $M$ is a Kripke model and $s$ is a world of $M$, called the {\em actual world}.
		For convenience, we assume that Kripke models are pointed.
	\end{definition}
	
	\begin{definition} \label{def:langkcSem} \rm
		Let $(M, s)$ be a Kripke model where $M = \km$.
		We interpret formulas in $\langkn$ by induction:
		\begin{itemize} 
			\item $M, s \models \true$;
			\item $M, s \models p$ if $p \in V(s)$;
			\item $M, s \models \neg \phi$ if $M, s \nmodels \phi$;
			\item $M, s \models \phi \land \psi$ if $M, s \models \phi$ and $M, s \models \psi$;
			\item $M, s \models \Know_i \phi $ if for all $t \in R_i(s)$, $M, t\models \phi$.
		\end{itemize}
	\end{definition}
	
	\looseness=-1
	We say $\phi$ is \textit{satisfiable}, if there is a model satisfying $\phi$;
	$\phi$ \textit{entails} $\psi$, written $\phi \models \psi$, if for any model $(M, s)$ satisfying $\phi$, $M, s \models \psi$;
	$\phi$ and $\psi$ are \textit{equivalent}, written $\phi \equiv \psi$, if $\phi \models \psi$ and $\psi \models \phi$.
	
	\looseness=-1
	Throughout this paper, we use $\lang$ and $\lang'$ for sublanguages of $\langkn$, and $\sublangprop$ and $\sublangpropP$ for propositional sublanguages.
	Throughout this paper, we assume that every propositional term and clause has a polynomial representation in $\sublangprop$ and $\sublangpropP$.
	All of the propositional sublanguages considered in \cite{DarM2002} obey with this assumption except the canonical $\DNF$.
	We say $\lang$ and $\lang'$ are \textit{dual}, if there is a polytime algorithm $f$ from $\lang$ to $\lang'$ s.t. for any formula $\phi \in \lang$, $f(\phi) \equiv \neg \phi$, and vice verse.
	For example, $\DNF$ and $\CNF$ are dual in propositional logic.
	
	\subsubsection{Normal forms}
	Cover disjunctive normal form \cite{CateCMV2006} and prime implicate normal form \cite{Bie2008} have been proposed for the description logic $\ALC$ that is a syntactic variant of $\Kn$.
	We rephrase them in $\Kn$.
	
	\begin{definition} \rm
		A formula $\phi$ is in \textit{cover disjunctive normal form ($\CDNF$)}, if it is generated by the BNF:
		\[\phi ::= \tau \land \bigwedge_{i \in \subagents} \cover_i \Phi_i \mid \phi \lor \phi, \]
		where $\tau$ is a satisfiable $\TE$, $\Phi_i$ are in $\CDNF$, $\subagents \subseteq \agents$, and $\cover_i \Phi_i$ is shorthand for $\Know_i (\bigvee_{\phi \in \Phi_i} \phi) \land \bigwedge_{\phi \in \Phi_i} \dualKnow_i \phi$.
	\end{definition}
	
	An epistemic clause\footnote{The definition of epistemic clauses in \cite{Bie2008} is slightly different from that in this paper.
		It is defined as a disjunction of propositional literals and epistemic literals.} $c$ is an \textit{implicate} of $\phi$, if $\phi \models c$.
	An epistemic clause $c$ is a \textit{prime implicate} of $\phi$, if $c$ is an implicate of $\phi$ and for all implicate $c'$ of $\phi$ s.t. $c' \models c$, $c \models c'$.	
	
	\begin{definition}\label{def:pinf} \rm
		A formula $\phi$ is in \textit{prime implicate normal form ($\PINF$)}, if it is $\true$ or $\false$, or satisfies the following:
		\begin{enumerate} 
			\item $\phi \nequiv \true$ and $\phi \nequiv \false$;
			\item $\phi$ is a conjunction $c_1 \land \cdots \land c_n$ of epistemic clauses where
			\begin{enumerate} 
				\item $c_j \nmodels c_k$ for $j \neq k$;
				\item each prime implicate of $\phi$ is equivalent to some conjunct $c_j$;
				\item every $c_j$ is a prime implicate of $\phi$ s.t. (i) if $d$ is a disjunct of $c_j$, then $c_j \nequiv c_j \setminus \set{d}$; (ii) $|\DiamSub[i](c_j)| \leq 1$ for $i \in \agents$; (iii) for every $i \in \agents$, if $\beta \in \BoxSub[i](c_j) \cup \DiamSub[i](c_j)$ then $\beta$ is in $\PINF$; (iv) for every $i \in \agents$, $\beta \in \BoxSub[i](c_j)$ and $\gamma \in \DiamSub[i](c_j)$, we have $\gamma \models \beta$.
			\end{enumerate}
		\end{enumerate}
	\end{definition}

	\subsubsection{Queries and transformations}
	\looseness=-1
	For a normal form considered in knowledge compilation, it is useful if it preserves major reasoning tasks and logical constructs (also referred to as queries and transformations).
	In this paper, we consider those queries and transformations, discussed in \cite{DarM2002} for propositional logic.
	Most of them can be directly generalized to multi-agent epistemic logics except modal counting ($\CT$) and enumeration ($\ME$) since any formula generally has infinitely many distinct models.
	
	\begin{definition}	\rm
		We say a language $\lang$ satisfies 
		\begin{itemize} 
			\item $\CO$ (resp. $\VA$), if there is a polytime algorithm deciding whether any formula $\phi \in \lang$ is satisfiable (resp. valid).
			
			\item $\EQ$ (resp. $\SE$), if there is a polytime algorithm deciding whether any formulas $\phi, \psi \in \lang$ satisfies the condition $\phi \equiv \psi$ (resp. $\phi \models \psi$).
			
			\item $\CE$ (resp. $\IM$), if there is a polytime algorithm deciding whether $\phi \models \psi$ (resp. $\psi \models \phi$) for any formula $\phi \in \lang$ and epistemic clause (resp. term) $\psi$.
			
			\item $\wedgeC$ (resp. $\veeC$), if there is a polytime algorithm generating a formula of $\lang$ equivalent to $\phi_1 \land \cdots \land \phi_n$ (resp. $\phi_1 \lor \cdots \lor \phi_n$) for every set $\set{\phi_1, \ldots, \phi_n}$ of $\lang$-formulas.
			
			\item $\wedgeBC$ (resp. $\veeBC$), if there is a polytime algorithm generating a formula of $\lang$ equivalent to $\phi \land \psi$ (resp. $\phi \lor \psi$) for any formulas $\phi, \psi \in \lang$.
			
			\item $\negC$, if there is a polytime algorithm generating a formula of $\lang$ equivalent to $\neg \phi$ for any formula $\phi \in \lang$.
		\end{itemize}
	\end{definition}
	
	\looseness=-1
	We now turn to another two important transformations: conditioning and forgetting.
	Conditioning is a syntactic operation defined as follows:
	
	\begin{definition} \label{def:conditioning} \rm
		Let $\phi \in \langkn$ and $\tau$ a satisfiable propositional term.
		The \textit{conditioning} of $\phi$ on $\tau$, written $\phi|_{\tau}$, is the formula obtained by replacing each variable $x$ of $\phi$ by $\true$ (resp. $\false$) if $x$ (resp. $\neg x$) is a positive (resp. negative) literal of $\tau$.
	\end{definition}
	
	\begin{definition}\label{def:CD} \rm
		A language $\lang$ satisfies $\CD$, if there is a polytime algorithm generating a formula of $\lang$ equivalent to $\phi|_{\tau}$ for every $\phi \in \lang$ and satisfiable propositional term $\tau$.
	\end{definition}
	
	\looseness=-1
	Intuitively, forgetting $\subprop$ from $\phi$ generates the logically strongest consequence of $\phi$ in which any variable of $\subprop$ does not appear.
	It can be applied in version control of knowledge bases and knowledge reuse.
	The definition of forgetting \cite{Fre2006} is given as follow. 
	
	\begin{definition} \label{def:forgetting} \rm
		Let $\phi \in \langkn$ and $\subprop \subseteq \prop$.
		We say $\psi$ is \textit{a result of forgetting $\subprop$ in $\phi$}, written $\exists \subprop. \phi$, if
		\begin{enumerate} 
			\item $\phi \models \psi$;
			\item $\prop(\psi) \subseteq \prop \setminus \subprop$;
			\item for any formula $\eta$ s.t. $\prop(\eta) \subseteq \prop \setminus \subprop$, $\phi \models \eta$ iff $\psi \models \eta$.
		\end{enumerate}
	\end{definition}
	
	\looseness=-1
	The result of forgetting is unique up to logical equivalence \cite{FangLD2016}.
	We hereafter use $\exists \subprop. \phi$ to denote the result of forgetting $\subprop$ in $\phi$.
	
	\begin{definition}\label{def:FO} \rm
		A language $\lang$ satisfies $\FO$ (resp. $\SFO$), if there is a polytime algorithm generating a formula of $\lang$ equivalent to $\exists \subprop. \phi$ (resp. $\exists \set{p}. \phi$) for any formula $\phi \in \lang$ and set of variables $\subprop$ (resp. variable $p$).
	\end{definition}

	\section{Separability-based $\DNF$ and $\CNF$}
	\looseness=-1
	In this section, based on logical separability, we introduce a general framework for defining normal forms $\DNF$ and $\CNF$ in $\Kn$.
	
	\looseness=-1
	One might define $\DNF$ for $\Kn$ as a disjunction of epistemic terms.
	However, this is not a proper definition for $\Kn$ due to lack of some desirable properties, such as the tractability for both satisfiability check and forgetting that propositional $\DNF$ supports.
	The issues distribute over disjunction, and thus the problem lies in the definition of epistemic terms as some epistemic terms are logically inseparable.
	Let us illustrate it in an example.
	
	\begin{example} \label{exm:kn1}
		Consider the formula $\phi = \Know_i (p \lor q) \land \Know_i(\neg p \lor q) \land \dualKnow_i \neg q$.
		The unsatisfiable formula $\dualKnow_i \false$ is not derived by any single epistemic literal of $\phi$.
		Deriving it requires reasoning about all conjuncts together.
		The satisfiability problem of epistemic terms cannot be decomposed into its conjuncts.
	\end{example}
	
	\looseness=-1
	This example illustrates that the polytime check for satisfiability holds for only logically separable epistemic terms. 
	
	\begin{definition}\label{def:logSepTerm} \rm
		Let $\phi$ be an epistemic term.
		We say $\phi$ is \textit{logically separable}, iff for every basic formula $\eta$, if $\phi \models \eta$, then there is $\alpha \in \propSub(\phi)$ or $\alpha$ is an epistemic literal that is a conjunct of $\phi$ s.t. $\alpha \models \eta$.
	\end{definition}
	
	Intuitively, logical separability requires that no logical puzzles are hidden within parts of epistemic terms.
	
	\begin{example} \label{exm:kn2}
		Continued with Example \ref{exm:kn1}, $\phi$ is logically inseparable since $\phi \models \dualKnow_i \false$ but no conjunct of $\phi$ entails $\dualKnow_i \false$.
		The formula $\psi = \Know_i q \land \dualKnow_i \false$, which is equivalent to $\phi$, is logically separable.
	\end{example}
	
	\looseness=-1
	Logical separable terms have the modularity property for satisfiability check and forgetting.
	The satisfiability problem of a logically separable epistemic term $\phi$ can be reduced to satisfiability subproblems of deciding whether each formula in $\PropSub(\phi)$ and $\DiamSub[i](\phi)$ is satisfiable.
	
	\begin{proposition} \label{prop:logsepSat}
		Let $\phi$ be a logically separable epistemic term.
		Then $\phi$ is satisfiable iff every formula $\alpha \in \PropSub(\phi) \cup \bigcup_{i \in \agents} \DiamSub[i](\phi)$ is satisfiable.
	\end{proposition}
	
	\looseness=-1
	Similarly, forgetting a set $Q$ of variables in $\phi$ can be accomplished by individually forgetting $Q$ in each formula of $\PropSub(\phi)$, $\BoxSub[i](\phi)$ and $\DiamSub[i](\phi)$.
	\begin{proposition} \label{prop:logsepForgetting}
		Let $\phi$ be a logically separable epistemic term and $Q$ a set of variables.
		Then 	\\
		$\exists Q. \phi \equiv \bigwedge_{\alpha \in \PropSub(\phi)} (\exists Q. \alpha) \land$ \\
		\hspace*{7.5mm} $\bigwedge_{i \in \subagents} [\bigwedge_{\beta \in \BoxSub[i](\phi)} (\Know_i (\exists Q.  \beta)) \land \bigwedge_{\gamma \in \DiamSub[i](\phi)} (\dualKnow_i (\exists Q. \gamma))]$.
	\end{proposition}
	
	To prove this property, we need a lemma.
	
	\begin{lemma}\label{lem:logSepTerm} \rm
		Let $\phi$ be a satisfiable logically separable epistemic term.		
		Then, the following statements hold:
		\begin{enumerate} \dense
			\item For each propositional formula $\alpha'$, $\phi \models \alpha'$ iff $\alpha \models \alpha'$ for some $\alpha \in \PropSub(\phi)$;
			
			\item For each $i \in \agents$ and each positive epistemic literal $\Know_i \beta'$, $\phi \models \Know_i \beta'$ iff $\beta \models \beta'$ for some $\beta \in \BoxSub[i](\phi)$;
			
			\item For each $i \in \agents$ and each negative epistemic literal $\dualKnow_i \gamma'$, $\phi \models \dualKnow_i \gamma'$ iff $\gamma \models \gamma'$ for some $\gamma \in \DiamSub[i](\phi)$.
		\end{enumerate}		
	\end{lemma}
	
	Now we give a proof for Proposition \ref{prop:logsepForgetting}.
	
	\begin{proof}
		\looseness=-1
		For brevity, we let $\psi$ be the right-hand-side formula.
		We consider two possible cases:
		
		\looseness=-1
		\textbf{Case 1.} $\phi$ is unsatisfiable:
		Then $\exists Q. \phi$ is also unsatisfiable. 
		By Proposition \ref{prop:logsepSat}, there is an unsatisfiable formula $\alpha \in \PropSub(\phi)$, or for some $i \in \agents$, there is $\gamma \in \DiamSub[i](\phi)$ s.t. $\gamma$ is unsatisfiable.
		Suppose that $\alpha$ is unsatisfiable.
		We get that $\psi$ is also unsatisfiable since $\psi$ contains an unsatisfiable conjunct $\exists Q. \alpha$.
		Similarly, $\psi$ is unsatisfiable in the case where $\gamma \in \DiamSub[i](\phi)$ is unsatisfiable.
		
		
		\looseness=-1
		\textbf{Case 2.} $\phi$ is satisfiable:
		Here we only verify the only-if direction for Condition 3 of Definition \ref{def:forgetting}: for any formula $\eta$ s.t. $\prop(\eta) \subseteq \prop \setminus \subprop$, if $\phi \models \eta$, then $\psi \models \eta$.		
		By the De Morgan's law, the distributive law of disjunction (resp. conjunction) over conjunction (resp. disjunction), and two transformation rules: $\neg \Know_i \phi \leftrightarrow \dualKnow_i (\neg \phi)$ and $\neg \dualKnow_i \phi \leftrightarrow \Know_i (\neg \phi)$, every $\langkn$-formula can be equivalently transformed into a conjunction of epistemic clauses.
		So we assume w.l.o.g. that $\eta$ is a conjunction of epistemic clauses.
		Let $c$ be a conjunct of $\eta$ and of the form $\bigvee_{\alpha' \in \PropSub(c)} \alpha' \lor \bigvee_{i \in \subagents'} [\bigvee_{\beta' \in \BoxSub[i](c)} (\Know_i \beta') \lor \bigvee_{\gamma' \in \DiamSub[i](c)} (\dualKnow_i \gamma')]$.
		It suffices to show that $\psi \models c$.
		For simplify, we let $\subagents = \subagents'$.
		Since $\phi \models c$, at least one of the following conditions holds.
		
		\begin{enumerate} \dense
			\item $\bigwedge_{\alpha \in \PropSub(\phi)} \alpha \land \bigwedge_{\alpha' \in \PropSub(c)} (\neg \alpha')$ is unsatisfiable;
			
			\item there exist $i \in \subagents$ and $\gamma \in \DiamSub[i](\phi)$ s.t. $\gamma \land \bigwedge_{\beta \in \BoxSub[i](\phi)} \beta \land \bigwedge_{\gamma' \in \DiamSub[i](c)} (\neg \gamma')$ is unsatisfiable;
			
			\item there exist $i \in \subagents$ and $\beta' \in \BoxSub[i](c)$ s.t. $\neg \beta' \land \bigwedge_{\beta \in \BoxSub[i](\phi)} \beta \land \bigwedge_{\gamma' \in \DiamSub[i](c)} (\neg \gamma')$ is unsatisfiable.
		\end{enumerate}
		
		\looseness=-1
		Here, we assume that Condition 2 holds.
		The other cases can be proven similarly.
		It follows that $\gamma \land \bigwedge_{\beta \in \BoxSub[i](\phi)} \beta \models \bigvee_{\gamma' \in \DiamSub[i](c)} \gamma'$.
		So $\dualKnow_i (\gamma \land \bigwedge_{\beta \in \BoxSub[i](\phi)} \beta) \models \dualKnow_i (\bigvee_{\gamma' \in \DiamSub[i](c)} \gamma')$.
		Since $\phi$ entails the former formula, we get that $\phi \models \dualKnow_i (\bigvee_{\gamma' \in \DiamSub[i](c)} \gamma')$.
		By Lemma \ref{lem:logSepTerm}, there is $\gamma^* \in \DiamSub[i](\phi)$ s.t. $\gamma^* \models \bigvee_{\gamma' \in \DiamSub[i](c)} \gamma'$.			
		Since $\exists Q. \gamma^*$ is the result of forgetting $Q$ in $\gamma^*$, we have $\gamma^* \models \bigvee_{\gamma' \in \DiamSub[i](c))} \gamma'$.
		Hence, $\dualKnow_i (\exists Q. \gamma^*) \models \dualKnow_i (\bigvee_{\gamma' \in \DiamSub[i](c))} \gamma')$, and $\psi \models \dualKnow_i (\bigvee_{\gamma' \in \DiamSub[i](c))} \gamma')$.
	\end{proof}
	
	
	\looseness=-1
	The following proposition gives the smallest logically separable epistemic term representation of an epistemic term $\phi$. 
	In this normal form, there is at most one propositional part, and at most one positive epistemic literal for each agent.
	Moreover, every formula inside $\dualKnow_i$ entails the corresponding formula inside $\Know_i$. 
	
	\begin{proposition} \label{prop:logsepSmallest}
		The smallest logically separable epistemic term representation of an epistemic term $\phi$ satisfies the following: 
		\begin{enumerate} 
			\item $|\PropSub(\phi)| \leq 1$;
			\item for each $i \in \agents$, $|\BoxSub[i](\phi)| \leq 1$;
			\item for each $i \in \agents$, $\beta \in \BoxSub[i](\phi)$ and $\gamma \in \DiamSub[i](\phi)$, $\gamma \models \beta$.
		\end{enumerate}
	\end{proposition}
	\begin{proof}
		It is trivial to prove the case where $\phi$ is unsatisfiable since the smallest representation of unsatisfiable formula is $\false$.
		We now assume that $\phi$ is satisfiable, and only verify Condition 1.
		The other two conditions can be proven similarly.		
		On the contrary, suppose that $\alpha_1, \alpha_2 \in \PropSub(\phi)$ but they are distinct.
		If $\alpha_1 \models \alpha_2$ or $\alpha_2 \models \alpha_1$, then one of them is redundant, and $\phi$ is not the most compact form.
		Otherwise, $\alpha_1 \nmodels \alpha_2$ and $\alpha_2 \nmodels \alpha_1$.
		Thus, neither $\alpha_1$ nor $\alpha_2$ entails $\alpha_1 \land \alpha_2$.
		This violates Lemma \ref{lem:logSepTerm}.
	\end{proof}
	
	\looseness=-1
	Forgetting in a logically separable epistemic term $\phi$ may not be tractably computed.
	This is because that some subformulas of $\phi$ may not be tractable for forgetting.
	To achieve polytime forgetting for logically separable epistemic terms, we need some further conditions on them.
	We not only require the logically separable epistemic term $\phi$ to be the smallest form, but also restrict the propositional part of $\phi$ to be in $\sublangprop$, and every formula of $\BoxSub[i](\phi)$ and $\DiamSub[i](\phi)$ to be the disjunction of formulas in this form.
	
	\begin{definition} \label{def:kterm} \rm
		An epistemic term $\phi$ is a \textit{separability-based term with $\sublangprop$} ($\KTE_{\sublangprop}$), if it is of the form $\alpha \land \bigwedge_{i \in \subagents} (\Know_i \beta_i \land \bigwedge_{j} \dualKnow_i \gamma_{ij})$
		s.t. 
		\begin{enumerate} 
			\item $\alpha \in \sublangprop$ and $\subagents \subseteq \agents$;
			\item $\beta_i$'s and $\gamma_{ij}$'s are disjunctions of $\KTE_{\sublangprop}$'s;
			\item $\gamma_{ij} \models \beta_i$ for any $i$ and $j$.
		\end{enumerate}
	\end{definition}	
	
	It is natural to obtain the definition of separability-based clauses that is dual to the notion of separability-based terms.
	
	\begin{definition}\label{def:kclause} \rm
		An epistemic clause $\phi$ is a \textit{separability-based clause with $\sublangprop$} ($\KCL_{\sublangprop}$), if it is of the form 
		$\alpha \lor \bigvee_{i \in \subagents} (\dualKnow_i \beta_i \lor \bigvee_{j} \Know_i \gamma_{ij})$
		s.t. 
		\begin{enumerate} 
			\item $\alpha \in \sublangprop$ and $\subagents \subseteq \agents$;
			\item $\beta_i$'s and $\gamma_{ij}$'s are conjunctions of $\KCL_{\sublangprop}$'s;
			\item $\beta_i \models \gamma_{ij}$ for any $i$ and $j$.
		\end{enumerate}
	\end{definition}
	
	We are ready to define separability-based $\DNF$ and $\CNF$. 
	
	\begin{definition} \label{def:kdnfcnf} \rm
		A formula $\phi$ is in \textit{separability-based disjunctive (resp. conjunctive) normal form with $\sublangprop$ ($\KDNF_{\sublangprop}$ (resp. $\KCNF_{\sublangprop}$))}, if $\phi$ is a disjunction (resp. conjunction) of $\KTE_{\sublangprop}$'s (resp. $\KCL_{\sublangprop}$'s).
	\end{definition}
	
	It is easily verified that two existing normal forms $\CDNF$ and $\PINF$ are sublanguages of $\KDNF$ and $\KCNF$ respectively.
	
	\begin{proposition} \label{prop:subsetLang}
		$\CDNF \subseteq \KDNF_{\TE}$ and $\PINF \subseteq \KCNF_{\CL}$.
	\end{proposition}
	\begin{proof}
		In the definition of CDNF (Definition 2.4), each $\cover_i \Phi_i$ is an STE since (1) it is shorthand for $\Know_i (\bigvee_{\phi \in \Phi_i} \phi) \land \bigwedge_{\phi \in \Phi_i} \dualKnow_i \phi$, and (2) $\phi \models \bigvee_{\phi \in \Phi_i} \phi$ for each $\phi$.
		Thus, CDNF is a fragment of SDNF.
		
		In the definition of PINF (Definition 2.5), Conditions 2-(c)-(ii) and -(iii) correspond to the form $\alpha \lor \bigvee_{i \in \subagents} (\dualKnow_i \beta_i \lor \bigvee_{j} \Know_i \gamma_{ij})$ and Condition (3) of the definition of SCL (Definition 3.3). So PINF is a fragment of SCNF.
	\end{proof}
	
	\section{Expressiveness and Succinctness}	
	\looseness=-1
	In this section, we analyze the expressive power and spatial complexity of the four normal forms.
	Our main results include: (1) the sizes of the $\KDNF$ and $\KCNF$ for a given formula are single-exponential in the size of the given formula, and (2) we provide a full picture of the succinctness for the four normal forms $\KDNF$, $\KCNF$, $\CDNF$ and $\PINF$.
	
	
	%
	\looseness=-1
	It is proven that every $\langkn$-formula is equivalent to a formula in $\CDNF$ (resp. $\PINF$) that is at most single (resp. double) exponentially large in the given formula size.
	This reflects that our new normal forms have a better space complexity than $\PINF$ and is at the same level as $\CDNF$.
	
	\begin{proposition} \label{prop:kdnfCnfTransUpperBound}
		Any formula in $\langkn$ is equivalent to a formula in $\KDNF_{\sublangprop}$ (or $\KCNF_{\sublangprop}$) that is at most single-exponentially large in the size of the original formula.
	\end{proposition}

	\begin{proof}			
		We only consider $\KDNF_{\sublangprop}$ as the case of $\KCNF_{\sublangprop}$ is similar.
		
		
		Let $\phi \in \langkn$.		
		We first transform $\phi$ into  $\phi'$ in NNF by pushing every negation symbol into variables and eliminating double negation symbols.	
		We then recursively transform $\phi'$ into an equivalent formula $\phi^*$ in $\KDNF_{\sublangprop}$ by induction on $\dep[\phi']$, the depth of nesting of epistemic operators.
		
		\noindent \emph{Base case}: If $\dep[\phi']=0$, $\phi$ is propositional and thus it can be equivalently transformed into a propositional DNF formula $\psi$. Then we obtain a formula $\phi^*$ in $\KDNF_{\sublangprop}$ by converting each disjunct of $\psi$ into $\sublangprop$.

		\noindent \emph{Inductive case}:		
		By the distributive law, we transform the formula $\phi'$ into a disjunction of epistemic terms $t$.
		For each term $t$, we first convert it into the form $\alpha' \land \bigwedge_{i \in \agents} (\Know_i \beta'_i \land \bigwedge_{j} \dualKnow_i \gamma'_{ij})$.
		The propositional formula $\alpha'$ is obtained by conjoining all propositional parts of $t$, \ie, $\alpha' = \bigwedge_{\alpha \in \PropSub(t)} \alpha$.
		In a similar way, we obtain the positive literal $\Know_i \beta'_i$ such that $\beta'_i = \bigwedge_{\beta_i \in \BoxSub[i](t)} \beta_i$.
		For each $\gamma_{ij} \in \DiamSub[i](t)$, we obtain a negative literal $\dualKnow_i \gamma'_{ij}$ where $\gamma'_{ij} = \gamma_{ij} \land \beta'_i$.
		By the inductive assumption, the subformulas $\beta'_i$ and $\gamma'_{ij}$ can be transformed into $\KDNF_{\sublangprop}$.
		
		We analyze the spatial complexity of this transformation.
		Firstly, the NNF formula $\phi'$ has size at most $2 \len[\phi]$ since only negation symbols are added, and there is at most one negation symbol for each occurrence of variables.
		Secondly, by induction on $\dep[\phi]$, we can show that $\phi \equiv \phi^*$ and that $\len[\phi^*]$ is single-exponential in $\len[\phi]$.
	\end{proof}
	
	\looseness=-1
	In the worst case, the number of prime implicates for a formula can be double-exponential in the size of the formula \cite{Bie2009}.
	The following proposition shows that the size of the smallest $\KDNF$ ($\KCNF$ and $\CDNF$) for a formula can be exponential in the worst case.	
	
	\begin{proposition} \label{prop:kdnfTransLowBound}
		Every $\KDNF_{\sublangprop}$ (resp. $\KCNF_{\sublangprop}$) formula equivalent to $\bigwedge_{j = 1}^{n} [(\Know_i p_j \land \dualKnow_i p_j) \lor (\Know_i p'_j \land \dualKnow_i p'_j)]$ (resp. $\bigvee_{j = 1}^{n} [(\Know_i p_j \lor \dualKnow_i p_j) \land (\Know_i p'_j \lor \dualKnow_i p'_j)]$) has at least $2^n$ epistemic terms (resp. clauses).
	\end{proposition}
	
	We now turn to compare the succinctness of the four normal forms.
	
	\begin{definition}\label{def:succ} \rm
		A language $\lang$ is \emph{at least as succinct as} $\lang'$, denoted $\lang \leq \lang'$, if there is a polynomial function $f$ from $\lang$ to $\lang'$ s.t. for any formula $\phi \in \lang'$, there exists a formula $\psi \in \lang$ s.t. $\psi \equiv \phi$ and $\len[\psi] \leq f(\len[\phi])$.
	\end{definition}
	
	\looseness=-1
	The following proposition indicates that the succinctness results for $\KDNF$ and $\KCNF$ can be reduced to the corresponding succinctness results in propositional logic.
	
	\begin{proposition} \label{prop:kdnfSucc}
		$\sublangprop \leq \sublangpropP$ iff $\KDNF_{\sublangprop} \leq \KDNF_{\sublangpropP}$ iff $\KCNF_{\sublangprop} \leq \KCNF_{\sublangpropP}$.
	\end{proposition}
	\begin{proof}
		\looseness=-1
		We only prove that $\sublangprop \! \leq \! \sublangpropP$ iff $\KDNF_{\sublangprop} \! \leq \! \KDNF_{\sublangpropP}$. 
		
		($\Rightarrow$):
		If $\sublangprop \leq \sublangpropP$, then there exists a mapping $t$ from $\sublangprop$ to $\sublangpropP$ satisfying two conditions:
		(1) $\phi \equiv t(\phi)$ for any formula $\phi \in \sublangprop$ and (2) there is a polynomial $f$ s.t. for each formula $\phi \in \sublangprop$, $\len[\phi] \leq f(\len[t(\phi)])$.
		
		We inductively construct a mapping $t'$ from $\KDNF_{\sublangpropP}$ to $\KDNF_{\sublangprop}$ as follows:
		\begin{itemize} 
			\item $t'(\phi) = t(\phi)$, if $\phi \in \sublangpropP$;
			\item $t'(\phi) = t(\alpha) \land \bigwedge_{i \in \subagents} [\Know_i t'(\beta_i) \land \bigwedge_{j} \dualKnow_i t'(\gamma_{ij})]$, \\
			if $\phi = \alpha \land \bigwedge_{i \in \subagents} (\Know_i \beta_i \land \bigwedge_{j \dualKnow_i \gamma_{ij}})$;
			\item $t'(\phi) = \bigvee_{i = 1}^n t'(\psi_i)$, if $\phi \notin \sublangpropP$ and $\phi = \bigvee_{i = 1}^n \psi_i$.
		\end{itemize}
		
		It is easily verified that $t'(\phi)$ is a formula in $\KDNF_{\sublangprop}$ such that $t'(\phi) \equiv \phi$ and $\len[\phi] \leq f(\len[t'(\phi)])$.
		
		($\Leftarrow$):
		On the contrary, assume that $\sublangprop \nleq \sublangpropP$.
		Let $\phi \in \sublangprop$ s.t. no equivalent formula $\phi'$ in $\sublangpropP$ satisfying the condition: $\len[\phi'] \leq f(\len[\phi])$ for any polynomial $f$.
		Obviously, $\phi$ is in $\KDNF_{\sublangprop}$.
		The smallest $\KDNF_{\sublangpropP}$ representation of propositional formula is an $\sublangpropP$-formula.
		Hence, $\KDNF_{\sublangprop} \nleq \KDNF_{\sublangpropP}$.
	\end{proof}
	
	\looseness=-1
	Table \ref{tab:suc} summarizes the results of succinctness for the four normal forms. 
	The symbol $\leq$ (or $\leq^*$) in the cell of row $r$ and column $c$ of Table \ref{tab:suc} means that ``the normal form $\lang_r$ given at column $r$ is at least as succinct as $\lang_c$ given at column $c$ (under the condition that $\lang_0 \leq \lang'_0$ in the case of $\leq^*$)".
	The symbol $\nleq$ means that ``$\lang_r$ is not at least as succinct as $\lang_c$".
	
	We make three observations from Table \ref{tab:suc}.
	First of all,  $\KDNF_{\sublangprop}$ (resp. $\KCNF_{\sublangprop}$) we propose are strictly more succinct than the existing normal form $\CDNF$ (resp. $\PINF$).
	In addition, $\KDNF$ and $\KCNF$ are incomparable w.r.t. succinctness.
	This incomparability relation also holds for the other three pairs of normal forms: ($\KDNF$, $\PINF$), ($\CDNF$, $\KCNF$), and ($\CDNF$, $\PINF$).
	Finally, $\CDNF$ and $\PINF$ are not at least as succinct as the other normal forms.
	
	\begin{table}
		\vspace*{-3mm}
		\small
		\centering
		\caption{Succinctness of normal forms in $\Kn$}
		\label{tab:suc}
		\begin{tabular}{| c | c | c | c | c |}
			\hline
			$\lang$                 & $\KDNF_{\sublangpropP}$ & $\KCNF_{\sublangpropP}$ & $\CDNF$     & $\PINF$  \\ \hline \hline
			$\KDNF_{\sublangprop}$ & $\leq^*$               & $\nleq$                & $\leq$      & $\nleq$ \\ \hline
			$\KCNF_{\sublangprop}$ & $\nleq$                & $\leq^*$               & $\nleq$     & $\leq$ \\ \hline
			$\CDNF$                 & $\nleq$                & $\nleq$                & $\leq$      & $\nleq$ \\ \hline
			$\PINF$                 & $\nleq$                & $\nleq$                & $\nleq$     & $\leq$ \\ \hline
		\end{tabular}			
		\vspace*{-3mm}
	\end{table}

	\begin{theorem}
		The results in Table \ref{tab:suc} hold.
	\end{theorem}
	\begin{proof} 
		We only prove that $\CDNF \nleq \KDNF_{\sublangprop}$.
		The other statements can be seen by Propositions \ref{prop:subsetLang} - \ref{prop:kdnfSucc}, the corresponding results for $\CDNF$ \cite{CateCMV2006} and $\PINF$ \cite{Bie2009},
		and the assumption that every term and clause has polynomial representation in $\sublangprop$ and $\sublangpropP$.		
		
		We define a class of formulas as follows:
		\begin{itemize} 
			\item $\phi_0 = p \lor q$;
			\item $\phi_k = \phi_0 \land \Know_i \phi_{k - 1}$.
		\end{itemize}
		Here $p$ and $q$ are propositional atoms. The size of $\phi_k$ is linear in $k$, more precisely, $3 + 5k$.
		Let $f$ be a polynomial s.t. any clause $c$ has a representation in $\sublangprop$ with size at most $f(|c|)$.
		Each $\phi_k$ has a polynomial representation $\phi'_k$ in $\KDNF_{\sublangprop}$ with size $f(3) \cdot (k + 1) + 2k$.
		The smallest representation in $\CDNF$ equivalent to $\phi_k$ is $(p \land \cover_i \set{\phi_{k - 1}}) \lor (p \land \cover_i \emptyset) \lor (q \land \cover_i \set{\phi_{k - 1}}) \lor (q \land \cover_i \emptyset)$.
		This formula has size single-exponential in $k$.
	\end{proof}
	
	\vspace*{-3mm}
	\begin{algorithm}
		\small
		\SetKwInOut{KwIn}{input}
		\SetKwInOut{KwOut}{output}
		\caption{$\SAT_{\Kn}(\phi)$}\label{alg:satKdnf}
		\KwIn{$\phi$: a formula in $\KDNF_{\sublangprop}$}
		\KwOut{Return $\true$ if $\phi$ is satisfiable, return $\false$ otherwise.}
		
		\uIf{$\phi \in \sublangprop$}
		{
			\Return{$\SAT_{\sublangprop}(\phi)$}.
		}
		\uElseIf{$\phi = \alpha \land \bigwedge_{i \in \subagents} (\Know_i \beta_i \land \bigwedge_{j} \dualKnow_i \gamma_{ij})$}
		{
			\Return{$\SAT_{\sublangprop}(\alpha) \land \bigwedge_{i \in \subagents} \bigwedge_{j} \SAT_{\Kn}(\gamma_{ij})$}.
		}
		\uElseIf{$\phi = \bigvee_j \psi_j$}
		{
			\Return{$\bigvee_j \SAT_{\Kn}(\psi_j)$}.
		}
	\end{algorithm}

	\vspace*{-3mm}
	\section{Queries and Transformations}	 
	
%
	
	\looseness=-1
	In this section, we mainly discuss $\KDNF_{\sublangprop}$ against the class of queries and transformations, and identify conditions of $\sublangprop$ under which some useful properties hold in $\KDNF_{\sublangprop}$.
	In particular, we give a tractable and modular algorithm for verifying the satisfiability of formulas in $\KDNF$.
	More importantly, we provide an almost complete picture for tractability of the four normal forms.
	These results, together with the results on succinctness, show that $\KDNF$ is the normal form most suitable for MAEP.

	\looseness=-1
	It is well-known that the satisfiability problem of $\DNF$ is tractable.
	This positive result is still valid for $\KDNF_{\sublangprop}$ if $\sublangprop$ allows polytime satisfiability check.
	Based on a given subprocedure $\SAT_{\sublangprop}$ for the satisfiability of $\sublangprop$, Algorithm \ref{alg:satKdnf} is the whole procedure that recursively decides if a $\KDNF$ formula $\phi$ is satisfiable via repeated application of the subprocedure.
	Due to the modularity property (cf. Proposition \ref{prop:logsepSat}), a logically separable epistemic term $\alpha \land \bigwedge_{i \in \subagents} (\Know_i \beta_i \land \bigwedge_{j} \dualKnow_i \gamma_{ij})$ is satisfiable iff all of $\alpha$ and $\gamma_{ij}$'s are satisfiable.
	Hence, the subprocedure $\SAT_{\sublangprop}$ is polytime, so is Algorithm \ref{alg:satKdnf}.
	Interestingly, even if the satisfiability problem of $\sublangprop$ is NP-Complete, the upper bound of the time complexity of Algorithm \ref{alg:satKdnf} falls into $\Delta^P_2$ since the number of propositional subformulas in $\phi$ is at most $|\phi|$, and this algorithm only calls for the subprocedure $\SAT_{\sublangprop}$ at most $|\phi|$ times.
	
	\begin{proposition} \label{prop:kdnfCO}
		If $\sublangprop$ satisfies $\CO$, then $\KDNF_{\sublangprop}$ satisfies $\CO$.
	\end{proposition}
	
	The negative results about other queries also carry forward from $\DNF$ to $\KDNF$.

	\begin{proposition} \label{prop:kdnfVASEEQIMCE}
		$\KDNF_{\sublangprop}$ does not satisfy $\VA$, $\SE$, $\EQ$, $\CE$ or $\IM$ unless $\P = \NP$.
	\end{proposition}
	\begin{proof}
		$\VA$:
		Let $\tau_1 \lor \cdots \lor \tau_n$ be a $\DNF$.
		For each $\tau_k$, there exists $\psi_k \in \sublangprop$ s.t. $\psi_k \equiv \tau_k$ and $|\psi_k| < f(|\tau_k|)$ for some polynomial $f$. 
		Clearly, $\psi_1 \lor \cdots \lor \psi_n$ is in $\KDNF_{\sublangprop}$.
		If we can decide whether this disjunction is valid in polytime, then the validity of $\DNF$ can be tractably accomplished.
		However, the latter problem is $\coNP$-complete. A contradiction.	
		
		$\SE$ and $\EQ$:
		Since $\SE$ implies $\VA$, $\KDNF_{\sublangprop}$ does not satisfy $\SE$.
		Similarly, $\KDNF_{\sublangprop}$ fails to satisfy $\EQ$.
		
		$\CE$ and $\IM$:
		Let $\Know_i \phi$ be an epistemic literal where $\phi$ is propositional.
		Clearly, $\true$ is in $\KDNF_{\sublangprop}$ and $\Know_i \phi$ is an epistemic term.
		We get that $\true \models \Know_i \phi$ iff $\phi$ is valid.
		The validity problem of propositional logic is $\coNP$-complete, and so is the problem that decides if $\true \models \Know_i \phi$.
		Hence, $\KDNF_{\sublangprop}$ does not satisfy $\CE$.
		Similarly, $\KDNF_{\sublangprop}$ fails to satisfy $\IM$.
	\end{proof}
	
	\looseness=-1
	Unlike $\DNF$, even if $\sublangprop$ satisfies the polytime clause entailment check ($\CE$), $\KDNF_{\sublangprop}$ does not possess such a property.
	Actually, it is impossible to propose a normal form permitting such a check.
	In the following, we will show that $\KDNF_{\sublangprop}$ supports a restricted polytime clausal entailment check after showing that $\KDNF_{\sublangprop}$ satisfies polytime bounded conjunction ($\wedgeBC$).
	
	By Definition \ref{def:kdnfcnf}, it is obvious that the disjunction of $\KDNF$ formulas can be generated efficiently.
	\begin{proposition} \label{prop:kdnfVee}
		$\KDNF_{\sublangprop}$ satisfies $\veeC$, and hence $\veeBC$.
	\end{proposition}
	
	Similarly to $\DNF$, $\KDNF$ is not closed under conjunction and negation.	
	
	\begin{proposition} \label{prop:kdnfWedgeCNegC}
		$\KDNF_{\sublangprop}$ does not satisfy $\wedgeC$ or $\negC$.
	\end{proposition}
	\begin{proof} 
		By Proposition \ref{prop:kdnfTransLowBound}, it follows that $\wedgeC$ does not hold in $\KDNF_{\sublangprop}$.
		On the contrary, assume that $\KDNF_{\sublangprop}$ satisfies $\negC$.
		This, together with Proposition \ref{prop:kdnfVee}, imply that $\KDNF_{\sublangprop}$ satisfy $\wedgeC$, a contradiction.
	\end{proof}

		\begin{table*}
		\vspace*{-3mm}	
		\scriptsize
		\centering
		\caption{Queries and transformations for normal forms in $\Kn$}
		\label{tab:ldnfcnfKC}
		\begin{tabular}{| c | c | c | c | c | c | c | c | c | c |  c | c | c | c | c | c | c | c |}
			\hline
			$\lang$              & $\CO$ & $\VA$ & $\SE$ & $\EQ$ & $\CE$ & $\CE_{\sublangpropP}$ & $\IM$ & $\IM_{\sublangpropP}$ & $\negC$ & $\wedgeC$ & $\wedgeBC$ & $\veeC$ & $\veeBC$ & $\CD$ & $\FO$ & $\SFO$  \\ \hline  \hline
			$\KDNF_{\sublangprop}$  & $\CO$          & $\circ$        & $\circ$        & $\circ$        & $\circ$        & $\CO, \wedgeBC$ & $\circ$        & $\circ$ & \xmark       & \xmark        & $\wedgeBC$          & \checkmark       & \checkmark        & $\CD$       & $\FO$       & $\SFO$                       \\ \hline
			$\KCNF_{\sublangprop}$  & $\circ$        & $\VA$          & $\circ$        & $\circ$        & $\circ$        &  $\circ$ & $\circ$        & $\VA, \veeBC$  & \xmark       & \checkmark         & \checkmark          & \xmark       & $\veeBC$          & $\CD$       & $\circ$     & ?           \\ \hline	
			$\CDNF$                    & \checkmark          & $\circ$        & $\circ$        & $\circ$        & $\circ$        & \checkmark & $\circ$        & $\circ$   & \xmark       & \xmark         & \checkmark          & \checkmark       & \checkmark        & \checkmark       & \checkmark       & \checkmark                     \\ \hline
			$\PINF$                    & \checkmark          & \checkmark     & \checkmark     & \checkmark     & $\circ$        & \checkmark & $\circ$        & \checkmark & \xmark       & \xmark     & \xmark     & \xmark     & \checkmark & ?                     & \checkmark        & \checkmark      \\ \hline				
		\end{tabular}		
		\vspace*{-3mm}
	\end{table*}
	
	Nonetheless, it supports bounded conjunction.
	
	\begin{proposition} \label{prop:kdnfWedgeBC}
		If $\sublangprop$ satisfies $\wedgeBC$, then $\KDNF_{\sublangprop}$ satisfies $\wedgeBC$.
	\end{proposition}
	\begin{proof}
		By assumption, there exists a polytime algorithm for generating an $\sublangprop$-formula $\alpha''$ equivalent to $\alpha \land \alpha'$ for each pair of formulas $\alpha, \alpha' \in \sublangprop$.
		Let $f$ be its time complexity, $k$ the degree of $f$, and $c$ the sum of the coefficients of $f$.
		So $f(\len[\alpha], \len[\alpha']) \leq c \len[\alpha]^{k} \len[\alpha']^{k}$.
		
		Given $\phi, \phi' \in \KDNF_{\sublangprop}$,
		we construct a formula $\phi''$ in $\KDNF_{\sublangprop}$ that is equivalent to $\phi \land \phi'$ by simply taking the disjunction $\psi''$ of all epistemic terms where $\psi'' \equiv \psi \land \psi'$ for each disjunct $\psi$ of $\phi$ and each disjunct $\psi'$ of $\phi'$.
		If $\len[\psi''] \leq c \len[\psi]^{k} \len[\psi']^{k}$ for every pair $\psi$ and $\psi'$, then $\len[\phi''] \leq c \len[\phi]^{k} \len[\phi']^{k}$.
		
		It remains to prove that $\len[\psi''] \leq c \len[\psi]^k \len[\psi']^k$.
		Let $\psi = \alpha \land \bigwedge_{i \in \subagents} (\Know_i \beta_i \land \bigwedge_{j = 1}^{m_i} \dualKnow_i \gamma_{ij})$ and $\psi' = \alpha' \land \bigwedge_{i \in \subagents'} (\Know_i \beta'_i \land \bigwedge_{j = 1}^{m'_i} \dualKnow_i \gamma'_{ij})$. 
		W.l.o.g., assume that $\subagents = \subagents'$.
		We construct a formula $\psi'' = \alpha'' \land \bigwedge_{i \in \subagents} (\Know_i \beta''_i \land \bigwedge_{j = 1}^{m_i} \dualKnow_i \gamma''_{ij} \land  \bigwedge_{j = 1}^{m'_i} \dualKnow_i \gamma^*_{ij})$,
		where $\alpha'' \equiv \alpha \land \alpha'$, $\beta''_i \equiv \beta_i \land \beta'_i$, $\gamma''_{ij} \equiv \beta'_i \land \gamma_{ij}$ and $\gamma^*_{ij} \equiv \beta_i \land \gamma'_{ij}$.
		It is easy to verify that $\psi''$ is an $\KTE$ with size at most $c \len[\psi]^{k} \len[\psi']^{k}$.
	\end{proof}
	
	\looseness=-1
	Now, we consider polytime tests for restricted clausal entailment and implicant.		
	If $\sublangprop$ satisfies both $\CO$ and $\wedgeBC$, and the epistemic clause is restricted to a $\KCL_{\sublangpropP}$, where $\sublangpropP$ is dual to $\sublangprop$, then $\KDNF_{\sublangprop}$ satisfies the polytime clause entailment check.
	
	\begin{definition}\label{def:CEIMKn} \rm
		A language $\lang$ satisfies $\CE_{\sublangprop}$ (resp. $\IM_{\sublangprop}$), if there is a polytime algorithm for deciding whether $\phi \!\models\! \psi$ (resp. $\psi \models \phi$) for every $\phi \in \lang$ and $\KCL_{\sublangprop}$ (resp. $\KTE_{\sublangprop}$) $\psi$.	
	\end{definition}
	
	\begin{proposition} \label{prop:kdnfCElang}
		Let $\sublangprop$ and $\sublangpropP$ be dual.
		If $\sublangprop$ satisfies $\CO$ and $\wedgeBC$, then $\KDNF_{\sublangprop}$ satisfies $\CE_{\sublangpropP}$.
	\end{proposition}
	\begin{proof}
		Let $\phi \in \KDNF_{\sublangprop}$ and $\psi \in \KCL_{\sublangpropP}$. 
		Deciding whether $\phi \models \psi$ can be accomplished as follows:
		(1) obtain an $\KTE_{\sublangprop}$ $\psi'$ equivalent to $\neg \psi$;
		(2) construct a formula $\phi'$ in $\KDNF_{\sublangprop}$ equivalent to $\phi \land \psi'$;
		(3) decide if $\phi'$ is satisfiable.
		If it is unsatisfiable, then $\phi \models \psi$; otherwise, $\phi \nmodels \psi$.
		By Propositions \ref{prop:kdnfCO} and \ref{prop:kdnfWedgeBC}, and the fact that $\KTE_{\sublangprop}$ and $\KCL_{\sublangpropP}$ are dual, the whole procedure is in polytime.
	\end{proof}
	
	\looseness=-1
	The normal form $\KDNF_{\sublangprop}$ still does not satisfy the restricted polytime implicant check even if $\sublangpropP$ is dual to $\sublangprop$.
	
	\begin{proposition} \label{prop:kdnfIMlang}
		Let $\sublangprop$ and $\sublangpropP$ be dual.
		Then $\KDNF_{\sublangprop}$ does not satisfy $\IM_{\sublangpropP}$ unless $\P = \NP$.
	\end{proposition}
	
	\looseness=-1
	It is easy to design procedures for generating the results of conditioning and forgetting of $\KDNF$ formulas respectively.
	They are similar to the procedure $\SAT$ for recursively deciding the satisfiability of $\KDNF$ formulas.
	For example, forgetting a variable $p$ in a $\KDNF$ formula $\phi$ can be computed by simply doing propositional forgetting on each propositional component (\eg, a maximal propositional subformula) of $\phi$.
	The next proposition states that conditioning and forgetting turn out to be tractable for $\KDNF_{\sublangprop}$ under some restrictions on $\sublangprop$.
	
	\begin{proposition} \label{prop:kdnfCDFOSFO}
		If $\sublangprop$ satisfies $\CD$ (resp. $\FO$/$\SFO$), then $\KDNF_{\sublangprop}$ satisfies $\CD$ (resp. $\FO$/$\SFO$).
	\end{proposition}

	\looseness=-1
	Since $\KCNF$ is dual to $\KDNF$, it is similar to obtain corresponding results for $\KCNF_{\sublangprop}$.
	Because $\CDNF$ is a sublanguage of $\KDNF_{\TE}$ and propositional term satisfies the corresponding conditions, so $\CDNF$ supports the queries which propositional $\DNF$ satisfies.
	Most similar results for $\PINF$ originate from \cite{DarM2002,Bie2009}.
	
	\looseness=-1
	Now we elaborate on the results for queries and transformations in Table \ref{tab:ldnfcnfKC}.
	The symbol $\checkmark$ in the cell of row $r$ and column $c$ of Table \ref{tab:ldnfcnfKC} means that ``the normal form $\lang_r$ given in row $r$ satisfies the polytime query (or transformation) property $P_c$ given in column $c$".
	Similarly, $\xmark$ means that ``$\lang_r$ does not satisfy $P_c$", $\circ$ means that ``$\lang_r$ does not satisfy $P_c$ unless $\P = \NP$", and ? means that ``the issue whether $\lang_r$ satisfies $P_c$ remains open".	
	For the query $\CE_{\sublangpropP}$, we require that $\sublangpropP$ is dual to $\sublangprop$ in $\KDNF_{\sublangprop}$ and $\KCNF_{\sublangprop}$, and that $\sublangpropP$ is $\CL$ in $\CDNF$ and $\PINF$.	
	For the query $\IM_{\sublangpropP}$, the requirement of $\sublangpropP$ is the same as that in $\CE_{\sublangpropP}$ except that $\sublangpropP$ is $\TE$ in the case of $\PINF$.
	$\KDNF_{\sublangprop}$ and $\KCNF_{\sublangprop}$ satisfy some query or transformation under certain conditions of $\sublangprop$.
	We list these conditions in the corresponding cell.
	For example, Proposition \ref{prop:kdnfCElang} says that ``if $\sublangprop$ satisfies both $\CO$ and $\wedgeBC$, then $\KDNF_{\sublangprop}$ satisfies $\CE_{\sublangpropP}$", and thus the cell of column 7 and row 2 is $\CO, \wedgeBC$.
	
	\looseness=-1
	We conclude this section by briefly summarising our main results.
	Given a suitable propositional sublanguage $\sublangprop$, $\KDNF_{\sublangprop}$ is tractable for all of queries and transformations that $\DNF$ admits, especially, $\CE_{\sublangpropP}$, $\wedgeBC$ and $\FO$ that are important for MAEP.
	$\CDNF$ satisfies the same properties as $\KDNF_{\sublangprop}$ does, but is less succinct than $\KDNF_{\sublangprop}$.
	$\KCNF$ does not satisfy polytime entailment check or forgetting.
	$\PINF$ is tractable for sentential entailment check ($\SE$) and forgetting, but it fails to satisfy $\wedgeBC$.
	From the knowledge compilation point of view, $\KDNF$ is more suitable for MAEP than the other three normal forms.
	
	\section{Application to MAEP}
	\looseness=-1
	\cites{BieFM2010} proposed a tractable approach to progression and entailment check for single-agent epistemic planning.
	It is challenging to extend their approach to multi-agent case.
	This is because that we need to consider not only first-order knowledge (\ie, to know what is the world), but also high-order knowledge, (\ie, to know what other agents know).
	
	\looseness=-1
	In this section, we explain how to apply our results in multi-agent epistemic planning.
	Especially, two essential procedures (progression and entailment check) can be accomplished efficiently by using the form $\KDNF$.

	
	\looseness=-1
	We begin with an MAEP domain adapted from \cite{KomG2015} to explain the MAEP task and the progression of actions, which will be used as a running example.
	\begin{example} \label{exm:roomAndBoxes}
		There are four rooms $r_1$, $r_2$, $r_3$ and $r_4$ in a row from left to right on a corridor.
		Each of two boxes $b_1$ and $b_2$ is located in a room.
		Two agents $i$ and $j$ can move from one room to its adjacent room.
		When an agent is in a room, she can sense if a box is in the room.	
		Initially, agent $i$ is in $r_1$, $j$ is in $r_4$, box $b_1$ is in $r_2$ and $b_2$ is in $r_3$.
		Each agent only knows where herself is. 
		The goal for agents $i$ and $j$ is to determine the rooms of $b_1$ and $b_2$.
	\end{example}
	
	\looseness=-1
	An MAEP task is formulated as follows.
	
	\begin{definition} \rm
		An MAEP task $\problem$ is a tuple $\tuple{\agents, \prop, \onticAct, \epiAct, \initialKB, \goal}$ where $\agents$ is a set of agents, $\prop$ is a set of variables, $\onticAct$ is a set of ontic actions, $\epiAct$ is a set of epistemic actions, $\initialKB$ is the initial KB, and $\goal$ is the goal formula.
	\end{definition}
	
	\begin{example} \label{exm:roomAndBoxes2}
		The domain in Example \ref{exm:roomAndBoxes} can be formalised into an MAEP task as follows:
		\begin{itemize} \dense
			\item Two agents: $i$ and $j$;
			\item Variables: $at(i, r)$, meaning agent $i$ is in room $r$; and $in(b, r)$, meaning box $b$ is in room $r$;
			\item Ontic actions: $left(i)$, agent $i$ moves left; and $right(i)$, agent $i$ moves right;
			\item Epistemic actions: $sense(i, b, r)$, agent $i$ senses whether box $b$ is in room $r$;
			\item The initial KB: $\Know_i at(i, r_1) \land \Know_j at(j, r_4)$; 
			\item The goal: $\Know_i in(b_1, r_2)] \land \Know_j in(b_2, r_3)$.  
		\end{itemize}
	\end{example}
	
	\looseness=-1
	In the following, we discuss the progression w.r.t. ontic and epistemic actions.
	
	\looseness=-1
	An ontic action $a_o$ is associated with a pair of functions $\tuple{\pre, \eff}$ where $\pre \in \langkn$ specifies the precondition and $\eff$ is the effect.
	In order to express the effect, we consider two versions $p$ and $p'$ of each variable $p$.
	For each $p \in \prop$, let the unprimed version $p$ stand for the fact that $p$ holds before performing the action $a_o$, and the primed one $p'$ stand for the fact that $p$ holds after.	
	The effect is a conjunction of formulas of the form:
	$p' \equiv \delta^+ \lor (p \land \neg \delta^-)$.	
	Two propositional formulas $\delta^+$ and $\delta^-$ are conditions that make $p$ true and false respectively.
	Intuitively, the effect means that $p$ holds after executing $a_o$ iff $\delta^+$ holds, or $\delta^-$ does not hold and $p$ holds initially.
	In Example \ref{exm:roomAndBoxes2}, if agent $i$ knows that  she is not in the rightmost room, then she can move right, and thus $\pre(right(i)) = \Know_i (\neg at(i, r_4))$; after moving right, agent $i$ will be in room $r_{n + 1}$ if she is in $r_{n}$ initially, and therefore $\eff(right(i)) = \bigwedge_{n = 1}^3 [at'(i, r_{n + 1}) \equiv at(i, r_{n})]$.
	

	\looseness=-1
	In this paper, we assume that all ontic actions are public and that there is no sort of imperfect information in them.	
	This assumption was proposed in \cite{KomG2015}. 
	To exactly capture progression under this assumption, it is necessary to progress all knowledge of agents according to the action effect via higher-order everyone knowledge.
	
	\begin{definition} \rm
		Let $k$ be a natural number and $\phi$ be a formula.
		The formula $\CKnow^k \phi$ is inductively defined:
		\begin{itemize} 
			\item $\CKnow^1 \phi = \bigwedge_{i \in \agents} \Know_i \phi$;
			\item $\CKnow^k \phi = \CKnow^{k - 1} \phi \land \bigwedge_{i \in \agents} \Know_i (\CKnow^{k - 1} \phi)$.
		\end{itemize}
	\end{definition}
	
	\looseness=-1 

	Intuitively, $\CKnow^1 \phi$ means that every agent knows that $\phi$ holds, \ie, $\phi$ is the everyone knowledge;
	and $\CKnow^k \phi$ that $\phi$ is the depth $k$ everyone knowledge.
	
	\looseness=-1
	The progression of the KB $\phi$ w.r.t. an ontic action $a_o$ can be accomplished as follows:
	\begin{enumerate} \dense
		\item Construct the formula $\psi$ by conjoining $\phi$ with the depth $k$ everyone knowledge about the effect of $a_o$ where $k$ is the depth of $\phi$: $\psi = \phi \land \CKnow^k \eff(a_o)$.
		\item Obtain the formula $\eta$ via forgetting the set $\subprop$ of unprimed version of primed variables in $\psi$, which occur in $\eff(a_o)$: $\eta = \exists \subprop. \psi$.
		\item Replace each occurrence of primed variables with their unprimed counterpart in $\eta$: $\eta[\prop'/\prop]$.
	\end{enumerate}
	
	\looseness=-1
	By expressing the initial KB and the ontic actions in $\KDNF_{\sublangprop}$ with $\sublangprop$ satisfies $\CO$, $\wedgeBC$ and $\FO$, which is always possible due to Proposition \ref{prop:kdnfCnfTransUpperBound}, the progression can be tractably computed in $\KDNF_{\sublangprop}$.
	
	
	\looseness=-1
	The progression of epistemic actions is relatively simple.
	An epistemic action $a_e$ is associated with a triple of functions $\tuple{\pre, \pos, \negation}$ of $\langkn$-formulas, where $\pre, \pos$ and $\negation$ indicate the precondition, the positive and negative sensing result,  respectively.
	For example, $\pre(sense(i, b_1, r_2)) \!=\! \Know_i at(i, r_2)$, $\pos(sense(i, b_1, r_2)) \!\! = \\ \Know_i in(b_1, r_2)$, and $\negation(sense(i, b_1, r_2)) = \Know_i \neg in(b_1, r_2)$.
	The computation of the progression is done via firstly making two copies of $\phi$, and then conjoining them with the positive and negative results respectively, \ie, $\phi \land \pos(a_e)$ and $\phi \land \negation(a_e)$.
	Again, $\KDNF$ is suitable to express such progression and supports polytime reasoning.
	
	Continued with Example \ref{exm:roomAndBoxes2}, we illustrate the computation of progression.
	\begin{example}
		Suppose that agent $i$ first moves right, and then senses whether $b_1$ in room $r_2$.		
		
		The progression of the initial KB $\initialKB$ w.r.t. the ontic action $right(i)$ is obtained as follows:
		\begin{enumerate} \dense
			\item Expand $\initialKB$ with $\CKnow^1 \eff(a_o)$, and the resulting formula is $\psi = \Know_i at(i, r_1)  \land \Know_j at(j, r_4) \land \CKnow^1 [\bigwedge_{n = 1}^3 (at'(i, r_{n + 1}) \equiv at(i, r_{n}))]$;
			
			\item Forget $\set{at(i, r_2), at(i, r_3), at(i, r_4)}$ in $\psi$, and get the formula $\Know_i at'(i, r_2) \land \Know_j at(j, r_4)$;
			
			\item Substitute $at'(i, r_2)$ with $at(i, r_2)$, which leads to $\phi' = \\ \Know_i at(i, r_2) \land \Know_j at(j, r_4)$.
		\end{enumerate}
		After agent $i$ moves right, she knows that her position is $r_2$.
		
		\looseness=-1
		Then, the progression of $\phi'$ w.r.t. $sense(i, b_1, r_2)$ contains the two KBs: $\Know_i [at(i, r_2) \land in(b_1, r_2)] \land \Know_j at(j, r_4)$ and $\Know_i [at(i, r_2) \land \neg in(b_1, r_2)] \land \Know_j at(j, r_4)$.
		These two formulas together means that agent $i$ knows whether $b_2$ is in $r_2$. 
	\end{example}
	
	\looseness=-1
	Finally, the task of MAEP is finding an action tree, whose branches on sensing results of epistemic actions and guarantees goal achievement after executing any path of actions.
	Besides progression, another major computation effort lies in the reasoning to decide if the current KB entails the goal formula and the preconditions of actions.
	By Proposition \ref{prop:kdnfCElang}, it follows that the entailment check is tractable if the current KB is in $\KDNF_{\sublangprop}$ and both the goal formula and the preconditions are in $\KCNF_{\sublangpropP}$, where $\sublangprop$ and $\sublangpropP$ are dual, and $\sublangprop$ satisfies $\CO$ and $\wedgeBC$.
	Since both progression and entailment check are tractable, the whole planning process can be done effectively.	

	\section{Extension to $\KFVn$}
	
	\begin{table}
		\vspace*{-3mm}
		\small
		\centering
		\caption{Succinctness of normal forms in $\KFVn$}
		\label{tab:sucK45n}
		\begin{tabular}{| c | c | c | c |}
			\hline
			$\lang$                 & $\KFVDNF_{\sublangpropP}$ & $\KFVCNF_{\sublangpropP}$ & $\ACDNF$ \\ \hline \hline
			$\KFVDNF_{\sublangprop}$ & $\leq^*$               & $\nleq$              & $\nleq$\\ \hline
			$\KFVCNF_{\sublangprop}$ & $\nleq$                & $\leq^*$              & $\nleq$ \\ \hline
			$\ACDNF$                 & $\nleq$                & $\nleq$                & $\leq$  \\ \hline
		\end{tabular}			
		\vspace*{-3mm}
	\end{table}
	
	\begin{table*}
		\vspace*{-3mm}
		\scriptsize
		\centering
		\caption{Queries and transformations for normal forms in $\KFVn$}
		\label{tab:queryAndTransK45n}
		\begin{tabular}{| c | c | c | c | c | c | c | c | c | c | c | c | c | c | c | c | c | c |}
			\hline
			$\lang$              & $\CO$ & $\VA$ & $\SE$ & $\EQ$ & $\CE$ & $\ACE_{\sublangpropP}$ & $\IM$ & $\AIM_{\sublangpropP}$ & $\negC$ & $\wedgeC$ & $\wedgeBC$ & $\veeC$ & $\veeBC$ & $\CD$ & $\FO$ & $\SFO$  \\ \hline  \hline
			$\KFVDNF_{\sublangprop}$  & $\CO$          & $\circ$        & $\circ$        & $\circ$        & $\circ$        & $\CO, \wedgeBC$ & $\circ$        & $\circ$  & \xmark       & \xmark        & $\wedgeBC$          & \checkmark       & \checkmark        & $\CD$       & $\FO$       & $\SFO$                      \\ \hline
			$\KFVCNF_{\sublangprop}$  & $\circ$        & $\VA$          & $\circ$        & $\circ$        & $\circ$        &  $\circ$ & $\circ$        & $\VA, \veeBC$ & \xmark       & \checkmark         & \checkmark          & \xmark       & $\veeBC$          & $\CD$       & $\circ$     & ?            \\ \hline	
			$\ACDNF$                    & \checkmark          & $\circ$        & $\circ$        & $\circ$        & $\circ$        & \checkmark & $\circ$        & $\circ$  & \xmark       & \xmark         & \checkmark          & \checkmark       & \checkmark        & \checkmark       & \checkmark       & \checkmark                      \\ \hline		
		\end{tabular}
		
		
		\vspace*{-3mm}
	\end{table*}
	
	\looseness=-1
	In the area of philosophy, it is ideal to assume that each agent has introspection about her own knowledge.
	This assumption can be captured by positive and negative introspection axioms $\axiomF$ ($\Know_i \phi \rightarrow \Know_i \Know_i \phi$) and $\axiomV$ ($\dualKnow_i \phi \rightarrow \Know_i \dualKnow_i \phi$).
	The former says that if agent $i$ knows $\phi$, then she know that she believes $\phi$, while the latter means that if agent $i$ does not know $\phi$, then she knows that she does not know $\phi$.
	A normal form for $\KFVn$ paves the way for developing normal forms for other logics, \eg, $\KDFVn$, which is more challenging to develop.
	In addition, it is better to use a logic with axioms $\axiomF$ and $\axiomV$ as the logical framework for MAEP.
	
	\looseness=-1
	In this section, we are concerned about the epistemic logic $\KFVn$, that contains both of positive and negative introspection axioms.
	The entailment and equivalence relations between formulas considered here are under $\KFVn$.
	
	\looseness=-1
	It is non-trivial to extend our proposed results for $\Kn$ to $\KFVn$ since the definition of separability-based term cannot be applied in $\KFVn$.
	We show this in an illustrative example.
	
	\begin{example}
		Suppose that $\phi = \dualKnow_i (p \land \Know_i \neg p)$.
		According to Algorithm \ref{alg:satKdnf}, $\phi$ is satisfiable in $\Kn$.
		However, it is not the case in $\KFVn$ since $\phi$ implies that $\dualKnow_i (p \land \neg p)$, which is equivalent to $\false$.
		This is due to the additional axioms $\axiomF$ and $\axiomV$.
	\end{example}
	
	\looseness=-1
	From the above example, we know that, in $\KFVn$, there exist logical entanglements between two propositional formulas on different depth of formulas.
	Hence, the crux is that some separability-based terms are logically inseparable in $\KFVn$.
	To achieve logical separability, we need to prohibits any consecutive occurrence of epistemic operators of the same agent.
	
	\begin{definition} \rm
		A formula has \textit{the alternating agent operator property} if no episteimc operators of an agent directly occur inside those of the same agent.
	\end{definition}
	
	\looseness=-1
	We say $\phi$ is an \textit{alternating separability-based term ($\KFVTE_{\sublangprop}$)}, if it is an $\KTE_{\sublangprop}$ with the alternating agent operator property.
	Similarly, we can define the following notions: \textit{alternating separability-based clause ($\KFVCL$), $\DNF$ ($\KFVDNF$), $\CNF$ ($\KFVCNF$) and cover $\DNF$ ($\ACDNF$)}.
	For example, the formula $\Know_i \dualKnow_i p$ is not an $\KFVTE$ since $\dualKnow_i$ occurs directly within the $\Know_i$ operator.
	But the formula $\Know_i \dualKnow_j \Know_i p$ is an $\KFVTE$ since there is a $\dualKnow_j$ operator inbetween two $\Know_i$ operators.
	
	%
	
	%
	
	\looseness=-1
	We remark that all results regarding succinctness, queries and transformations, stated in Sections 4 and 5, also hold for $\KFVDNF$, $\KFVCNF$ and $\ACDNF$ in the logic $\KFVn$ just as $\KDNF$, $\KCNF$ and $\CDNF$ in $\Kn$ except 
	transforming into $\KFVDNF$ or $\KFVCNF$ causes an at most double exponential in the size of the original formula.
	
	\begin{proposition} \label{prop:kfvdnfTrans}
		In $\KFVn$, every formula in $\langkn$ is equivalent to a formula in $\KFVDNF_{\sublangprop}$ or $\KFVCNF_{\sublangprop}$ that is at most double-exponentially large in the size of the original formula.
	\end{proposition}
	\begin{proof}
		We only consider $\KFVDNF_{\sublangprop}$ as the case of $\KFVCNF_{\sublangprop}$ can be similarly proven.
		
		The transformation is similar to that illustrated in Proposition \ref{prop:kdnfCnfTransUpperBound}.
		It contains one more step to ensure the alternating agent operator property.
		Let $\phi \in \langkn$.
		The transformation consists of three steps:
		(1) Put $\phi$ into an equivalent NNF formula $\phi'$;
		(2) Obtain the formula $\phi''$ with alternating agent operator property;
		(3) Get the formula $\phi^*$ by recursively transforming $\phi''$ into $\KFVDNF_{\sublangprop}$ by induction on $\dep[\phi'']$.
		
		The details of Steps 1 and 3 are shown in Proposition \ref{prop:kdnfCnfTransUpperBound}.
		Step 2 strips out any occurrence of consecutive epistemic operators with the same agent via the following
		equivalences in $\KFVn$:
		\begin{enumerate} \dense
			\item $\Know_i (\phi \lor (\Know_i \psi \land \eta)) \leftrightarrow (\Know_i \phi \lor \Know_i \psi) \land \Know_i(\phi \lor \eta)$;
			\item $\Know_i (\phi \lor (\dualKnow_i \psi \land \eta)) \leftrightarrow (\Know_i \phi \lor \dualKnow_i \psi) \land \Know_i(\phi \lor \eta)$;
			\item $\Know_i (\phi \land (\Know_i \psi \lor \eta)) \leftrightarrow \Know_i \phi \land (\Know_i \psi \lor \Know_i \eta)$;
			\item $\Know_i (\phi \land (\dualKnow_i \psi \lor \eta)) \leftrightarrow \Know_i \phi \land (\Know_i \eta \lor \dualKnow_i \psi)$.
		\end{enumerate}
		Step 3 preserves the alternating agent operator property, and thus the resulting formula is in $\KFVDNF_{\sublangprop}$.		
		
		Finally, we analyze the complexity of this transformation.
		Recall that the complexity analysis in Proposition \ref{prop:kdnfCnfTransUpperBound}, Step 1 leads to the formula $\phi'$ with the size at most $2 \len[\phi]$, and Step 3 causes an at most single-exponential blowup.
		In addition, Step 2 generates the formula $\phi''$ with size at most single-exponential in $\len[\phi']$.
		In summary, the whole conversion causes an at most double-exponential blowup.
	\end{proof}

	\looseness=-1
	Although the aforementioned transformation for arbitrary formulas may cause a double-exponential blowup, its complexity falls into single-exponential if we require that the original formula has the alternating agent operator property.
	
	\looseness=-1
	Secondly, we slightly adjust the definition of polytime tests for restricted clausal entailment ($\ACE_{\sublangprop}$) and implicant ($\AIM_{\sublangprop}$) by using $\KFVTE$ and $\KFVCL$ instead of $\KTE$ and $\KCL$ respectively.
	Similar to Proposition \ref{prop:kdnfCElang}, if $\sublangprop$ and $\sublangpropP$ are dual, then $\KFVDNF_{\sublangprop}$ and $\KFVDNF_{\sublangprop}$ satisfies $\KFVCL_{\sublangpropP}$ and $\KFVTE_{\sublangpropP}$ respectively.

	\looseness=-1
	Finally, the results regarding succinctness, queries and transformations of normal forms for $\KFVn$ are summarized in Tables 3 and 4.
	\begin{theorem}
		The result in Tables 3 and 4 hold.
	\end{theorem}

\section{Conclusions}
\looseness=-1
We have introduced a notion of logical separability for epistemic terms, which is a key property to guarantee that the satisfiability check and forgetting can be computed in modular way.
Based on the logical separability, we have defined a normal form $\KDNF$ for the multi-agent epistemic
logic $\Kn$, which can be seen as a generalization of the well-known propositional normal form $\DNF$.
As a dual to $\KDNF$, we can define the $\KCNF$ for $\Kn$.
More importantly, we have constructed a knowledge compilation map on four normal forms $\KDNF$, $\KCNF$, $\CDNF$ and $\PINF$ in terms of their succinctness, queries and transformations.
Interestingly, bounded conjunction, forgetting and restricted clausal entailment check are all tractable for $\KDNF_{\sublangprop}$ formulas under some restrictions on $\sublangprop$.
These three properties are crucial to effective implementations of MAEP.
Although $\KDNF$ and $\CDNF$ admit tractability for certain kind of entailments, $\KDNF$ is a better choice of the target compilation language for MAEP since the former is strictly more succinct than the latter.
Finally, by resorting to the alternating agent operator property, we extend our results to the epistemic logic $\KFVn$.

\looseness=-1
In future work, we plan to implement an effective multi-agent epistemic planner based on $\KDNF$.
It is also interesting to identify tractable normal forms in other multi-agent epistemic logics, \eg, $\KDFVn$ and $\SVn$.
Since the description logic $\ALC$ is highly-related to $\Kn$, the results proposed in this paper is also applicable to $\ALC$.
Another direction is to investigate knowledge compilation in more expressive description logics.

	\bibliographystyle{aaai}
	\bibliography{KR-2018}

\begin{thebibliography}{}

\bibitem[\protect\citeauthoryear{Bienvenu, Fargier, and
  Marquis}{2010}]{BieFM2010}
Bienvenu, M.; Fargier, H.; and Marquis, P.
\newblock 2010.
\newblock {Knowledge Compilation in the Modal Logic S5}.
\newblock In {\em Proceedings of the Twenty-Fourth AAAI Conference on
  Artificial Intelligence (AAAI-2010)},  261--266.

\bibitem[\protect\citeauthoryear{Bienvenu}{2008}]{Bie2008}
Bienvenu, M.
\newblock 2008.
\newblock {Prime Implicate Normal Form for $\mathcal{ALC}$ Concepts}.
\newblock In {\em Proceedings of the Twenty-Third AAAI Conference on Artificial
  Intelligence (AAAI-2008)},  412--417.

\bibitem[\protect\citeauthoryear{Bienvenu}{2009}]{Bie2009}
Bienvenu, M.
\newblock 2009.
\newblock {\em {Consequence Finding in Modal Logic}}.
\newblock Ph.D. Dissertation, Universit{\'e} de Toulouse.

\bibitem[\protect\citeauthoryear{Bryant}{1986}]{Bry1986}
Bryant, R.~E.
\newblock 1986.
\newblock {Graph-Based Algorithms for Boolean Function Manipulation}.
\newblock {\em IEEE Transactions on Computers} 100(8):677--691.

\bibitem[\protect\citeauthoryear{Darwiche and Marquis}{2002}]{DarM2002}
Darwiche, A., and Marquis, P.
\newblock 2002.
\newblock {A Knowledge Compilation Map}.
\newblock {\em Journal of Artificial Intelligence Research} 17:229--264.

\bibitem[\protect\citeauthoryear{Fang, Liu, and van
  Ditmarsch}{2016}]{FangLD2016}
Fang, L.; Liu, Y.; and van Ditmarsch, H.
\newblock 2016.
\newblock {Forgetting in Multi-Agent Modal Logics}.
\newblock In {\em Proceedings of the Twenty-Fifth International Joint
  Conference on Artificial Intelligence (IJCAI-2016)},  1066--1073.

\bibitem[\protect\citeauthoryear{French}{2006}]{Fre2006}
French, T.~N.
\newblock 2006.
\newblock {\em {Bisimulation Quantifiers for Modal Logics}}.
\newblock Ph.D. Dissertation, University of Western Australia.

\bibitem[\protect\citeauthoryear{Hales, French, and Davies}{2012}]{HalFD2012}
Hales, J.; French, T.; and Davies, R.
\newblock 2012.
\newblock {Refinement Quantified Logics of Knowledge and Belief for Multiple
  Agents}.
\newblock In {\em Proceedings of the Ninth Conference on Advances in Modal
  Logic (AiML-2012)},  317--338.

\bibitem[\protect\citeauthoryear{Halpern and Moses}{1992}]{HalM1992}
Halpern, J.~Y., and Moses, Y.
\newblock 1992.
\newblock A guide to completeness and complexity for modal logics of knowledge
  and belief.
\newblock {\em Artificial Intelligence} 54:319--379.

\bibitem[\protect\citeauthoryear{Huang \bgroup et al\mbox.\egroup
  }{2017}]{HuangFWL2017}
Huang, X.; Fang, B.; Wan, H.; and Liu, Y.
\newblock 2017.
\newblock {A General Multi-agent Epistemic Planner Based on Higher-order Belief
  Change}.
\newblock In {\em Proceedings of the Twenty-Ninth International Joint
  Conference on Artificial Intelligence (IJCAI-2017)},  3327--3334.

\bibitem[\protect\citeauthoryear{Kominis and Geffner}{2015}]{KomG2015}
Kominis, F., and Geffner, H.
\newblock 2015.
\newblock Beliefs in multiagent planning: From one agent to many.
\newblock In {\em Proceedings of the Twenty-Fifth International Conference on
  Automated Planning and Scheduling (ICAPS-2015)},  147--155.

\bibitem[\protect\citeauthoryear{Levesque}{1998}]{Lev1998}
Levesque, H.
\newblock 1998.
\newblock {A Completeness Result for Reasoning with Incomplete First-Order
  Knowledge Bases}.
\newblock In {\em Proceedings of the Sixth International Conference on
  Principles of Knowledge Representation and Reasoning (KR-1998)},  14--23.

\bibitem[\protect\citeauthoryear{Muise \bgroup et al\mbox.\egroup
  }{2015}]{MuiBFMMPS2015}
Muise, C.; Belle, V.; Felli, P.; McIlraith, S.; Miller, T.; Pearce, A.~R.; and
  Sonenberg, L.
\newblock 2015.
\newblock {Planning Over Multi-Agent Epistemic States: A Classical Planning
  Approach}.
\newblock In {\em Proceedings of the Twenty-Ninth AAAI Conference on Artificial
  Intelligence (AAAI-2015)},  3327--3334.

\bibitem[\protect\citeauthoryear{ten Cate \bgroup et al\mbox.\egroup
  }{2006}]{CateCMV2006}
ten Cate, B.; Conradie, W.; Marx, M.; and Venema, Y.
\newblock 2006.
\newblock {Definitorially Complete Description Logics}.
\newblock In {\em Proceedings of the Tenth International Conference on
  Principles of Knowledge Representation and Reasoning (KR-2006)},  79--89.

\end{thebibliography}

\end{document}